\documentclass[letterpaper]{article} 
\usepackage{aaai23}  
\usepackage{times}  
\usepackage{helvet}  
\usepackage{courier}  
\usepackage[hyphens]{url}  
\usepackage{graphicx} 
\usepackage{natbib}  
\usepackage{caption} 
\usepackage{algorithm}
\usepackage{algorithmic}

\usepackage{tikz,pgfplots}

\usepackage{newfloat}
\usepackage{listings}
\DeclareCaptionStyle{ruled}{labelfont=normalfont,labelsep=colon,strut=off} 
\floatstyle{ruled}
\newfloat{listing}{tb}{lst}{}
\floatname{listing}{Listing}
\title{City-scale Pollution Aware Traffic Routing\\ by Sampling Max Flows using MCMC}
\author {
    Shreevignesh Suriyanarayanan, 
    Praveen Paruchuri,  
    Girish Varma 
}
\affiliations {
    Machine Learning Lab, IIIT Hyderabad\\
    sshreevignesh.s@research.iiit.ac.in, praveen.p@iiit.ac.in, girish.varma@iiit.ac.in
}

\usepackage{amsmath}
\usepackage{amssymb}
\usepackage{amsfonts}
\usepackage{booktabs}
\usepackage{caption} 
\usepackage{graphicx}
\usepackage{mathtools}
\usepackage{amsthm}
\usepackage{newfloat}
\usepackage{tikz}
\usepackage{array,multirow,graphicx}

\begin{document}
\theoremstyle{plain}
\newtheorem{theorem}{Theorem}
\newtheorem*{proposition}{Proposition}
\newtheorem{lemma}[theorem]{Lemma}
\newtheorem{corollary}[theorem]{Corollary}
\theoremstyle{definition}
\newtheorem*{definition}{Definition}

\newtheorem{assumption}[theorem]{Assumption}
\theoremstyle{remark}
\newtheorem{remark}[theorem]{Remark}

\maketitle

\begin{abstract}A significant cause of air pollution in urban areas worldwide is the high volume of road traffic. Long-term exposure to severe pollution can cause serious health issues. One approach towards tackling this problem is to design a pollution-aware traffic routing policy that balances multiple objectives of i) avoiding extreme pollution in any area ii) enabling short transit times, and iii) making effective use of the road capacities. We propose a novel sampling-based approach for this problem. We provide the first construction of a \emph{Markov Chain} that can sample integer max flow solutions of a planar graph, with theoretical guarantees that the probabilities depend on the aggregate transit length. We designed a traffic policy using diverse samples and simulated traffic on real-world road maps using the SUMO traffic simulator. We observe a considerable decrease in areas with severe pollution when experimented with maps of large cities across the world compared to other approaches. 
\end{abstract}

\section{Introduction}
Long-term exposure to high amounts of air pollution causes various health issues \cite{pope2006health}. Data from WHO \cite{world2016ambient} shows that $91\%$ of the world's population lives in places where the pollution levels exceed the guideline limits. Outdoor air pollution accounts for an estimated $4.2$ million deaths per year, primarily due to stroke, heart disease, lung cancer, and chronic respiratory diseases. Low and middle-income countries suffer the most, especially in the Western Pacific and South-East Asia regions. Road traffic is considered one of the most significant contributors to air pollution in urban environments \cite{gualtieri2015statistical}. Studies have shown that people are willing to choose greener routes when credible information is provided \cite{AHMED2020101965}.

\begin{figure}[t]
  \centering
  
    \resizebox{0.8\columnwidth}{!}{
\includegraphics{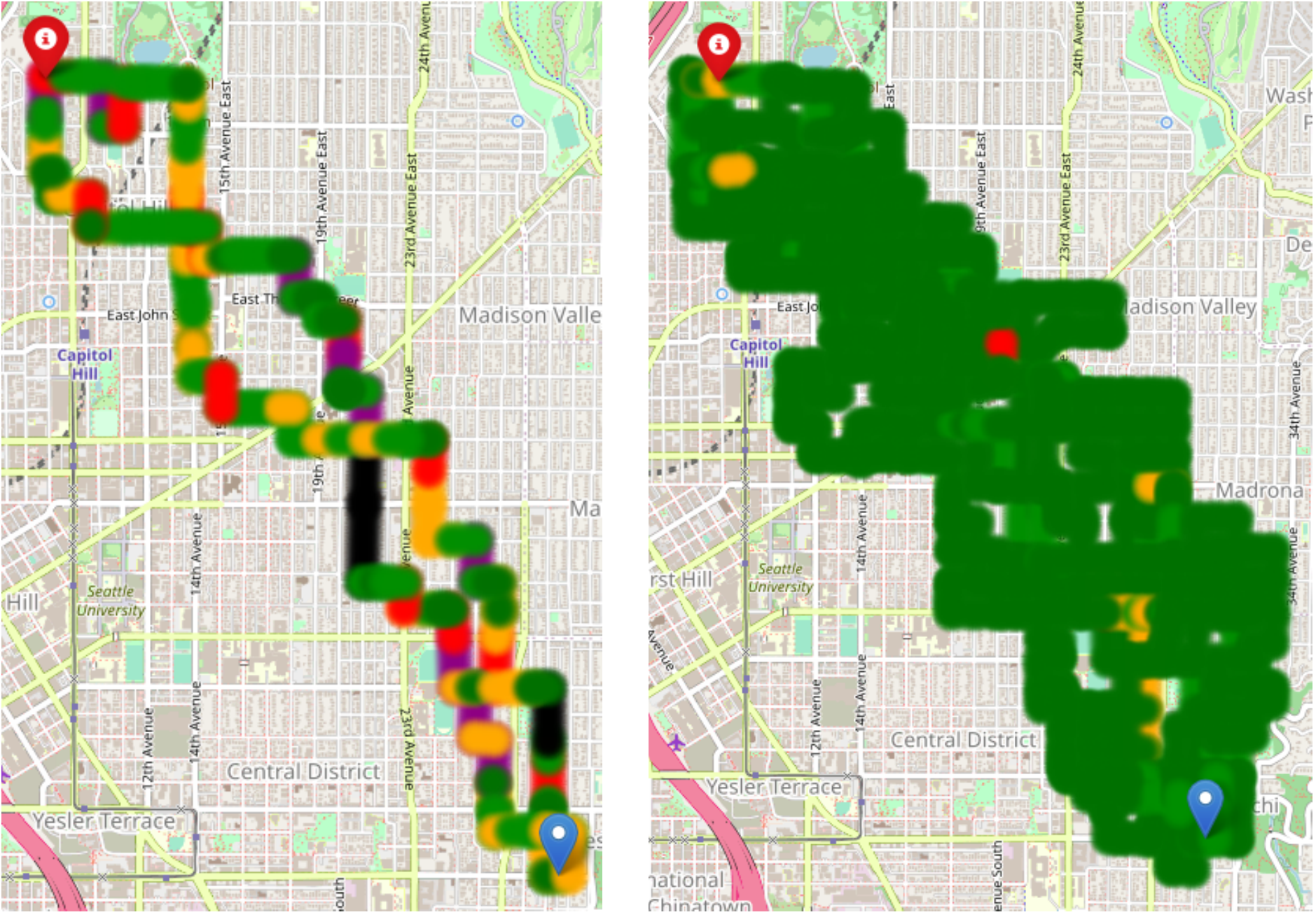}}
\resizebox{\columnwidth}{!}{
\includegraphics{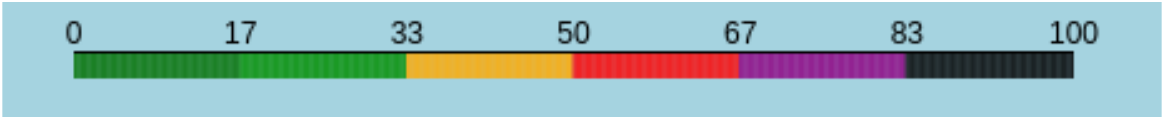}}
    \caption{Pollution released on road links for traffic simulation on a large area of Seattle, USA. Colour legend given at bottom is the percentage of the max pollution. Left is obtained using only a single flow (\emph{FFA}) and right is using our \emph{MaxFlow-MCMC} which uses multiple diverse max flow solutions. As can be seen severe pollution (red, purple, black) is prevented.}
    \label{fig:front-pic}
\end{figure}%



\citet{KAMISHETTY2020102194} propose the development of a transportation policy that distributes the traffic flow more evenly through a city to reduce the concentration of pollution in specific pockets. They use the concept of $k$-optimality, which ensures that any two flows have at most $\leq k$ edges in common, and design a traffic policy using multiple \emph{k-optimal maximum flow solutions} to route traffic differently on different days or for reasonable timelines. Here, \emph{k-optimality} is a measure \cite{pearce2007quality} that is used to capture distinctness between solutions while \emph{maximum flow} solutions \cite{ford1956maximal} are generated to maintain the throughput of traffic network. However, existing work in this space can only be used for small areas due to computational scalability issues. 
\begin{figure*}
  \centering
  \resizebox{0.8\textwidth}{!}{
  \includegraphics{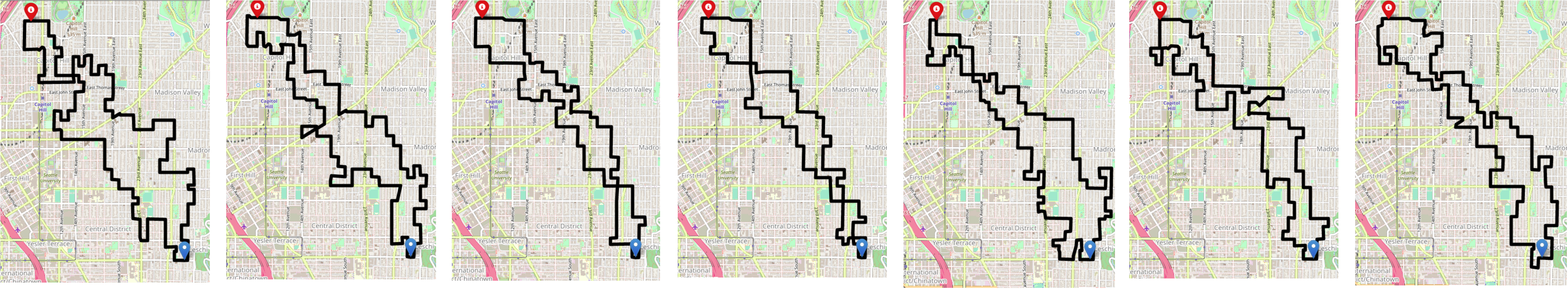}}
  \caption{Seven diverse max flows samples from the \emph{MaxFlow-MCMC} algorithm. Traffic is routed according to these and simulated in \emph{SUMO} for equal intervals, resulting in the emission heatmap shown in Figure \ref{fig:front-pic}.}
  \label{fig:Individual_sols}
\end{figure*}
We build upon this line of work to develop a significantly faster solution for pollution aware traffic routing using a \emph{Markov Chain Monte Carlo (\emph{MCMC})} \cite{Diaconis,Hastings, Rosenbluth, Bubley} based method to sample integer max flow solutions from a planar graph and generate a k-Optimal set of max flow solutions (hence named \emph{MaxFlow-MCMC}). Our \emph{Markov chain} extends the algorithm for sampling paths in planar graphs \cite{MontanariPenna}, and we provide proof of convergence to the stationary distribution, which assigns a higher probability to max flow solutions with shorter aggregate path length. 

To compare our algorithm with the previous work, we simulated the algorithm on real-world road networks of multiple cities. We use \emph{Simulation of Urban Mobility (SUMO)} \cite{SUMO2018}, which is a traffic simulator in combination with \emph{OpenStreetMap (OSM)} \cite{haklay2008openstreetmap}, an open-source tool that helps to model traffic settings on a real-world map. We also use the emission modeling capability of SUMO to evaluate the pollution levels generated by the different solution approaches, namely \emph{MaxFlow-MCMC}, \emph{k-optimal Pareto Max Flow Algorithm (k-PMFA)} \cite{KAMISHETTY2020102194} and \emph{Ford-Fulkerson Algorithm (FFA)} \cite{ford1956maximal}. Our results show that \emph{MaxFlow-MCMC} performs as well as the \emph{k-PMFA} algorithm in terms of pollution severities while being scalable to large-scale cities. In particular, we obtained a $79\%$ decrease in normalized mean pollution compared to \emph{FFA} on a larger map corresponding to Seattle (with $18699$ edges).\\

\noindent
\textbf{Our Contributions.} 
\begin{enumerate}
    \item We design a \emph{Markov Chain} that can be used to sample integer max flow solutions from a planar graph with probabilities proportional to an exponential of the length of the solution (see Section \ref{sec:MCMC}).
    \item We use the \emph{MCMC} method to obtain a set of diverse max flows (k-optimal) and use it to give the first approach to designing a large-scale transport policy that avoids severe pollution while maintaining reasonable transit times. Traffic routing using such a set of max flows ensures that i) a diverse set of paths are used, which prevents concentration of pollution ii) vehicles are routed through shorter paths, and iii) capacity of the road network is fully utilized as it is a max flow (see Section \ref{sec:traffic_policy}).
    \item We test our \emph{MaxFlow-MCMC} traffic policies using the \emph{SUMO} traffic simulator on real-world maps.The performance was compared with \emph{k-PMFA} and \emph{FFA} algorithms. We also demonstrate the scalability of \emph{MaxFlow-MCMC} on large real-world road maps (see Section \ref{sec:results}).
\end{enumerate}

\section{Related Works}

\begin{figure*}[t]
    \centering

    \resizebox{0.82\textwidth}{!}{
    \includegraphics{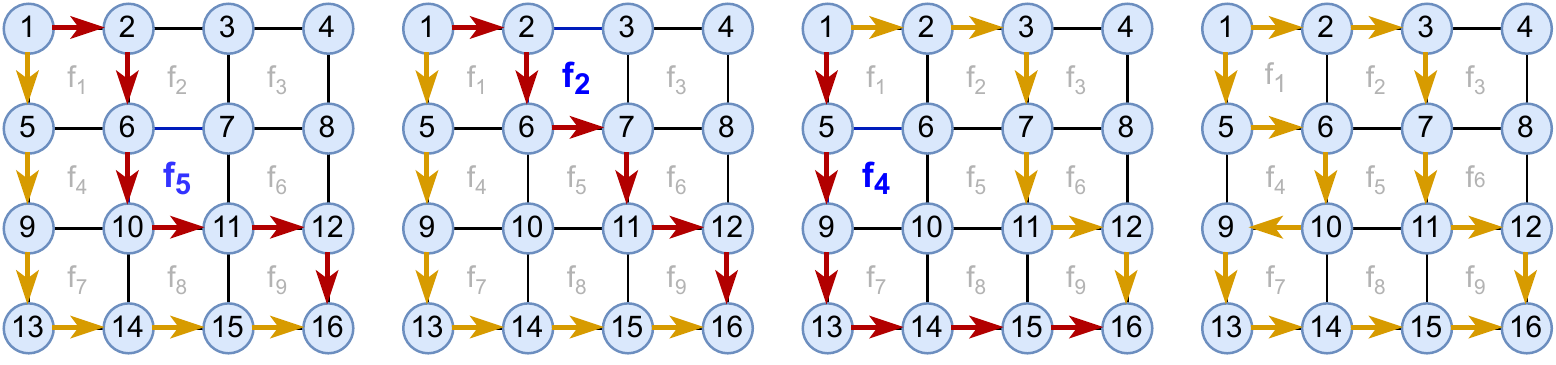}}
    \caption{Example of a transition in the \emph{Markov chain}. The underlying road network has 16 junctions forming a grid with max flow value from $1$ to $16$ to be $2$. The state space consist of all max flows from $1$ to $16$, which can also be considered as a pair of $1 - 16$ paths (marked using arrows). The transition in $M_{\text{flow}}$ given in Algorithm \ref{algo:mflow_step}, involves choosing one of the paths $p$ (marked in red) and one of the faces $f \in \{f_1,\cdots f_{10} \}$ (marked as blue) of the planar graph at random and rerouting the path along the face. The faces $\{f_1\cdots f_{9}\}$ are marked in the graph and $f_{10}$ is the outer face.
    }
    \label{fig:transitions}
\end{figure*}

\paragraph{Maximum Flow Problem.}
The \emph{Ford-Fulkerson} algorithm \cite{ford1956maximal}, is a popular algorithm to compute the maximum flow between two given points in a network, including traffic networks \cite{schrijver2002history}. 
However, it provides us with only a single max flow solution. We aim to develop an algorithm that will provide multiple max flow solutions between two points.

\paragraph{Pollution aware routing.}
There has been some work on the topic of pollution aware routing \cite{ALAM201882}, \cite{boriboonsomsin2012eco}. However, most of the previous work is focused on minimising the exhaust emissions and fuel consumption of a vehicle. In recent times, Google Maps has introduced eco-friendly routing, which suggests a path in a similar manner \cite{3newways27:online}. Our work focuses on improving the health of residents who live near frequently used roads by reducing their long term exposure to high levels of pollution i.e., we do not explicitly aim to reduce the pollution but aim to distribute the pollution better.



\paragraph{k-opt Pareto Max Flow Algorithm (k-PMFA).}
The \emph{k-Opt Pareto Max Flow} Algorithm \cite{KAMISHETTY2020102194}, computes a set of k-Optimal max flow solutions where each max flow solution has at most $k$ common edges with every other max flow solution. One of the steps in the algorithm involves computing all the simple paths from the source to the destination. A simple path between two nodes is defined as a path where no node appears more than once \cite{easley2012networks}. Counting all the simple paths from a source to a destination was proved to be \#P complete \cite{valiant1979complexity}, which implies that finding all the simple paths would be an NP-hard problem. Our algorithm generates a k-optimal set without requiring us to find and store all the simple paths, and hence it will be more efficient in terms of time and memory requirements.

\paragraph{MCMC Method for sampling paths in a planar graph.}
There is extensive literature on using \emph{MCMC} method for generating combinatorial structures \cite{JVV1986}. Typically such generation gives a uniform distribution over structures. In our case, we need to sample from a non-uniform distribution \cite{MartinRandall}. 
\citet{MontanariPenna} designs an \emph{MCMC} method for sampling simple paths in a planar graph. We construct a \emph{Markov Chain} over integer Max Flows of a graph, with specific guarantees on stationery distribution. Our approach for sampling max flow solutions from a planar graph is inspired by them but requires additional proofs for showing the irreducibility of the \emph{Markov chain}, which ensures that capacity constraints are not violated.

\section{Sampling Integer Max Flows using MCMC}
\label{sec:MCMC}

We consider the road network to be a planar graph $G(V,E)$. Edges $E$ of the graph will be the roads, while nodes $V$ will be the junctions where roads intersect or represent the end of roads. Each road in the network has a length and a specific number of lanes used as the edge's capacity since one vehicle can travel through each lane at a particular time. Let $s$ be the source node from which the vehicles travel to the destination node $t$. We want to ensure that the maximum number of vehicles can travel from $s$ to $t$ at any time and hence represent this as a maximum flow problem.

Our approach starts with an initial integer max flow solution and makes small random modifications to it, resulting in a sample from a distribution over integer max flow solutions. This method is commonly known as \emph{Monte Carlo Markov Chain (MCMC)} method. The \emph{Ford Fulkerson Algorithm} can find the initial max flow or set of paths. The distribution will result in higher probabilities for max flows whose aggregate length of paths is shorter.

 Since we are designing a system to route vehicles, we focus on integer max flow solutions where the flow through each edge is a non-negative integer. The \emph{Integrality Theorem} \cite{erickson1999algorithms}, which is a corollary of the \emph{Ford-Fulkerson Algorithm}, states that there exists at least one integer max flow solution in a graph where all the edge capacities are non-negative integers. Hence, we will have at least one integer max flow solution for a road network.

\subsection{A Markov Chain for Integer Max Flow}
Let us represent an integer max flow solution as a set of $mf$ number of paths, where $mf$ is the maximum flow value from $s$ to $t$, and each path contains a flow of $1$. Note that if an edge has a flow value $f$, it will be part of exactly $f$ paths. We define a \emph{Markov Chain} $M_{\text{flow}}$, whose state space $\Omega$ is the set of all integer max flow solutions. We define the total length of a max flow solution $x$ denoted $|x|$ as the sum of the lengths of all the paths in it. The length of each path will be the sum of the lengths of all the edges in the path. We can reroute a path of the max flow solution using a face if they have one or more common edges, as shown in Figure \ref{fig:transitions}. These transitions are inspired by \cite{MontanariPenna}, who built a similar \emph{Markov Chain} for sampling paths. The transitions rules of $M_{\text{flow}}$ are given in Algorithm \ref{algo:mflow_step}. 

The algorithm has a hyperparameter $\lambda$ that decides the preference given to the total length of a state. If $\lambda < 1$, transitions to a state with a lesser total length will be more probable. Similarly, if $\lambda > 1$, transitions to a state with a higher total length will be more probable. When $\lambda=1$, 
the transition to all neighbouring paths will have the same probability. 


\subsection{Convergence of the Markov Chain}
We show that the \emph{Markov chain} $M_{\text{flow}}$ converges to a stationary distribution which has the property that the probability of max flow $x$ will be proportional to $\lambda^{|x|}$. Setting $\lambda <1$, allows us to sample max flows with shorter aggregate lengths. First, we show that the distribution with the above property is a stationary distribution.


\begin{lemma}\label{lem:stationery}
\emph{A stationary distribution of $M_{\text{flow}}$ is
$$ \pi(x) = \frac{\lambda^{|x|}}{Z} \qquad \text{ where } \qquad Z = \sum_{y \in \Omega}\lambda^{|y|}$$
}
\end{lemma}
\begin{proof}
The proof uses standard arguments and is provided in appendix. 
\end{proof}





\begin{algorithm}[t]
    \caption{$M_{\text{flow}}(x)$ defines a step of the \emph{Markov Chain}  on current state $x$. See Figure \ref{fig:transitions}, for an example of its run. }
    \label{algo:mflow_step}
    \begin{algorithmic}[1] 
    \STATE $\mbox{faces} =$ set of all faces in the planar graph
    \STATE $mf$ = the current integer max flow
    \STATE $\mbox{paths} =$ set of paths in $x$ 
    \STATE $b \gets \mbox{Uniform}(\{0,1\})$
    \IF {$b == 1$}
    \STATE f $\gets$ Uniform(faces), p $\gets$ Uniform(paths)\;
    \IF {$f,p$ do not share an edge \textbf{or} rerouting $p$ through $f$ violates capacity}
    \STATE \textbf{return} $x$
    \ELSE
    \STATE $y \gets \mbox{reroute}(p,f)$
    \STATE \textbf{return} $y$ with probability $\min\{1,\frac{\lambda^{|y|}}{\lambda^{|x|}}\}$  where $|x|$ and $|y|$ are the total lengths of $x$ and $y$. 
    \STATE \textbf{return} $x$
    \ENDIF
    \ENDIF
    \end{algorithmic}
    \end{algorithm}


Next, we show that starting from any state, $M_{\text{flow}}$ converges to $\pi$, by showing that it is Ergodic. An Ergodic \emph{Markov Chain} will have a unique stationary distribution and will converge to it \cite{durrett2019probability}. For Ergodicity, we need to show that it is i.) aperiodic and ii.) irreducible.

\noindent
\textbf{Aperiodicity.} A \emph{Markov chain} is said to be aperiodic if 
$$ \forall x \in \Omega, \mbox{gcd}\{t \in \mathbb{N} | P ^{t} (x, x) > 0\} = 1.$$
where $P^t$ is the $t$ step transition probabilities.
For $M_{\text{flow}}$, we defined transitions in such a way that there is a self-loop with a probability of at least $1/2$. Therefore, $M_{\text{flow}}$ is aperiodic. 

\noindent
\textbf{Irreducibility.} A \emph{Markov chain} is said to be irreducible if
$$ \forall x, y \in \Omega, \exists  t \in N \mid P_{t}(x, y) > 0.$$

That is, every state in state space can be reached from every other
state with a finite number of steps.  It is shown in \citet{MontanariPenna} that
    we can reach any simple s-t path from any other path by rerouting through the faces of the planar graph.
    However, in our case, we need to ensure that the capacity constraints are not violated.  We show this using the following definition and proposition. 


\begin{definition}[Outer paths of an Integer Flow]
We assume that the line joining $s$ and $t$ to be the direction of positive $x$ axis. We define the \emph{outer path} of an integer flow in a planar graph as a path between $s$ and $t$ with a non-zero flow, such that there is no other path in the flow with a non-zero flow in between the path and the outer face of the graph. We define the \emph{top outer path} as the outer path that contains the node with the highest y-coordinate and the \emph{bottom outer path} as the one that contains the node with the least y-coordinate. 
\end{definition}

\begin{proposition}
We can reach any state in $M_{\text{flow}}$ from any other state in a finite number of steps.
\end{proposition}
\begin{proof}
    Let $x$ and $y$ be any two states in $\Omega$ and $\textit{paths}(x), \textit{paths}(y)$ be the set of $mf$ paths in $x, y$ respectively. Let $\textit{op}_x \in \textit{paths}(x)$ and $\textit{op}_y \in \textit{paths}(y)$ be the top outer paths of $x$ and $y$ respectively. We transform $\textit{op}_x$ and $\textit{op}_y$ to the same $s$ - $t$ path using a sequence of reversible $M_{\text{flow}}$ transitions.Let $s,v_1 \cdots, v_r,t$ be common nodes of $\textit{op}_x$ and $\textit{op}_y$. For each of the $r+1$ subpaths between these nodes, we convert the subpath in the lower one of $\textit{op}_x,\textit{op}_y$ to the other. This is possible without violating capacity constraints since there is no flow above the top outer paths by definition. Then we consider the residual flow $x', y'$ given by removing the unit flow through these paths from $x,y$. Then we repeat the same process on $x',y'$, until the residual flow becomes $0$. A diagrammatic example of the proof is given in the appendix for clarity.

\end{proof}



Since our \emph{Markov chain} is both irreducible and aperiodic, it is ergodic and will have a unique stationary distribution to which it will converge irrespective of the initial state. As proved in Lemma \ref{lem:stationery}, in the stationary distribution, the probability of a max flow solution $x$ is proportional to $\lambda^{|x|}$. 


\section{Traffic Routing using MaxFlow-MCMC}
\label{sec:traffic_policy}

We propose to use the \emph{MaxFlow-MCMC} algorithm (see Algorithm \ref{alg:mf-mcmc}) to generate a k-Optimal max flow solution set, which can be used for routing vehicles in such a way that the pollution is evenly spread out. Even though our chain is not rapidly mixing, we can start sampling before the chain has completely mixed as we do not need the solutions to follow the specific distribution for our use case since we are just sampling to generate a k-optimal set of max flow solutions. We define the k-optimality of our solution set similar to the definition of k-similarity in \citet{barrett2008engineering}. A set of max flow solutions is k-optimal when any two solutions from the set have at most $k$ common edges.

We use the \emph{Ford-Fulkerson Algorithm} \cite{ford1956maximal} to find one max flow solution for the starting state. We initially let the \emph{Markov chain} run for a few iterations without sampling to ensure that the initial state is random when we start sampling. We sample a random max flow solution from the current distribution every few iterations and check if it is k-optimal with the solution set. If it is, we append it to the solution set. We keep repeating this until we reach the exit condition. The \emph{MaxFlow-MCMC} Algorithm has multiple parameters we can modify to suit our time and solution constraints, making it scalable for larger maps.

\subsection{Implementing Policy in Simulation and Real World}

There are multiple ways by which the \emph{MaxFlow-MCMC} algorithm can be implemented in real world scenarios: (a) As described in \citet{KAMISHETTY2020102194}, traffic police can use a different solution to route the traffic on different days. For example, by identifying a set of 7 solutions, traffic can be routed using one solution per day of the week. A solution can be implemented by restricting the set of paths available for a vehicle to take and the number of vehicles on each path but does not restrict a vehicle's choice among the allowed ones. (b) GPS software(s) like Google maps can use to suggest different set of paths to different users on different days \cite{3newways27:online}. (c) Autonomous cars can be routed using our method, where vehicles can be instructed to use different paths.

\subsection{Traffic Simulation using SUMO}
\emph{Simulation of Urban MObility (SUMO)} is an open source, microscopic and continuous traffic simulation package designed to handle large networks. SUMO comes with a large set of tools for scenario creation and simulation and can handle multiple aspects of traffic flow generation, including computation of acceleration, deceleration, emission modeling, congestion, the distance between vehicles, etc., resulting in realistic modeling.

\begin{algorithm}[t]
    
    \begin{algorithmic}[1] 
    \STATE $x \gets$ FFA(s,t), solutionset $\gets \emptyset$
    \WHILE{iter $\leq$ \texttt{num\_iter} + \texttt{mix\_iter}}
    \IF {iter $>$ \texttt{mix\_iter}  and iter\% $sf$ == 0}
    \IF {KOpt($x$,solutionset)}
    \STATE solutionset $\gets$ solutionset + $x$
    \ENDIF
    \ENDIF
    \STATE x $\leftarrow M_{\text{flow}}(x)$ (See Algorithm 1) 
    \STATE iter $\gets$ iter+1
    \ENDWHILE
    \STATE \textbf{return} solutionset
    \end{algorithmic}
    \caption{ $x$ denotes current state of the \emph{Markov chain}, \emph{FFA} denotes Ford Fulkerson Algorithm, \texttt{mix\_iter} denotes number of iterations for mixing, \texttt{num\_iter} denotes number of iterations for sampling and $sf$ denotes sampling frequency.}
    \label{alg:mf-mcmc}
    \end{algorithm}
    
\subsection{Parameters for MaxFlow-MCMC}

\textbf{Sampling Frequency}: The number of iterations after which we sample a max flow solution from the current distribution. If we want to increase the number of solutions in the solution set, we can sample the solutions more frequently, and if the aim is to improve the runtime of the algorithm, we can sample less frequently.\\
\noindent
\textbf{Exit Condition}: The User can decide the exit condition of the algorithm based on requirements. The algorithm can be run for a fixed number of iterations, or it can be run until a solution set of the required size is obtained.\\
\noindent
\textbf{Lambda($\lambda$)}: Affects the total length of the max flow solutions. A lesser value of lambda would mean that the solutions obtained will typically have a lesser total path length. In comparison, a higher value of lambda would mean that the solutions will also include paths of higher length (which can allow sampling a large number of solutions). \\
\noindent
\textbf{Value of k}: Increasing the Value of $k$ would decrease the time required to obtain a set of solutions of the same size, but the set would have more common edges, leading to a lesser spread of pollution in general.

\subsection{Complexity Analysis of MaxFlow-MCMC}

\paragraph{Time Complexity.} The initial \emph{FFA} step is of complexity $O(|E| * mf)$, where $|E|$ is the number of edges and $mf$ is the max flow. Calculating the faces needs $O(kV)$ on average, and $O(V^2)$ in the worst-case \cite{schneider2019finding}. At every step of the algorithm, the following computations need to be done: 
(a) Randomly choosing a face needs $O(|F|) = O(|E|)$ since $|F|+|E| -|V| = 2$ by Euler's formula \cite{trudeau2013introduction}, where $|F|$ and $|V|$ are the number of faces and nodes respectively. 
(b) Checking for max flow condition needs reasoning over the number of edges in the path, i.e., $O(|E|)$. 
(c) Rerouting involves $O(|E|)$. 
(d) Checking for k-Optimality needs to be done every $sf$ steps needing $O(\frac{num\_sol*|E|^2}{sf})$. 

The overall worst-case time complexity would therefore be $O(\frac{num\_iter  * num\_sol * |E|^2}{sf} + |V|^2)$. Therefore, worst case runtime of \emph{MaxFlow-MCMC} will be polynomial in $|E|$ and $|V|$, provided that $num\_iter$ and $num\_sol$ are not very high both of which can be tuned per need in the algorithm. Note that the runtime for \emph{MaxFlow-MCMC} does not include the time required for calculating the faces of the graph since it was done as a prepossessing step. \citet{schneider2019finding} proves that even this step can be computed with an average complexity of $O(k|V|)$ and therefore will not make a significant change to the overall runtime.

\paragraph{Space Complexity.} Following are the key steps involved in the algorithm, which require: (a) A memory of $O(|E|^2)$ for storing all the faces. (b) A memory of $num\_sol * mf *|E|)$ for storing the solution set. This leads to a overall worst-case space complexity of $O((num\_sol*mf*|E| + |E|^2)$. In summary, our algorithm has a time and space complexity that is tractable in the number of edges and, therefore, scalable to larger networks.

\section{Experimental Details}
\label{sec:expt-det}

\begin{table}[t]
    \centering
    \begin{tabular}{lrrr}
    \toprule
        Name & $\lambda$ & $num\_iter$ & $sf$  \\
        \midrule
        MaxFlow-MCMC1 & 0.95 & 50000 & 25 \\
        MaxFlow-MCMC2 & 0.95 & 25000 & 25 \\
        MaxFlow-MCMC3 & 0.90 & 50000 & 25 \\
        MaxFlow-MCMC4 & 0.95 & 50000 & 50 \\
        MaxFlow-MCMCLarge & 0.99 & until 7 sol & 25 \\
        MaxFlow-MCMCKanpur & 0.95 & until 7 sol & 25 \\
    \bottomrule    
    \end{tabular}
    
    
    \caption{Parameter values used for experiments. The first four are used for ablation studies and the last two are for large scale experiments.} 
    \label{tab:maxlfowmcmcproperties}
    
\end{table}

The parameters for experiments on large scale graphs named \emph{MaxFlow-MCMC Large}, are presented in Table \ref{tab:maxlfowmcmcproperties}. Since the \emph{MaxFlow-MCMC} method is applicable only for planar graphs, we modified the real-world network to make it planar by removing the non-planar components such as bridges. Since there can be cases of multiple sources and destinations for vehicles, we performed experiments with more than one source and destination also. We introduce a virtual source and a virtual destination node, connected to all the source and sink nodes respectively, using infinite capacity edges \cite{6108162}.
Algorithms were then used with vehicles traveling from the virtual source to the virtual sink. 


We tested \emph{MaxFlow-MCMC} for different values of $k$ and simulated traffic using the \emph{SUMO} simulator \cite{SUMO2018}. We used Sage Math \cite{sagemath} to find the faces of the roadmap. Vehicles were released in waves with an interval of $5$ seconds each for the small map and an interval of $15$ seconds each for the larger map. They are released in waves so that a reasonable distance is created before the next wave of vehicles is allowed into the simulation. The number of vehicles per wave was equal to the maximum flow from the source to the destination. We increased the interval for the large map since it involved roads with different speed limits. Vehicles also have to slow down at junctions for turning, and a larger path can potentially involve a larger number of turns. We use $NOx$ values (i.e., $NO$ and $NO_{2}$ values) to measure the pollution due to their high dependency on traffic flow and their adverse impact on human health \cite{tomas2013analysis}. The plots were generated using matplotlib \cite{Hunter:2007}.

Due to the randomness involved in our algorithm, different runs using the same parameters may provide us with a different number of solutions. Hence, we use each solution of the generated k-optimal set for $1$ hour and compute \emph{Normalised Mean pollution} ($NOx_{nm}$) to compare the different runs of the algorithm accurately. 
$$NOx_{nm} = \frac{ \textit{Mean amount of NOx released per hour} }{ \textit{Total number of edges with non-zero emissions} }$$

\emph{SUMO} provides us with the amount of NOx released for every edge with non-zero emissions over the simulation period. We calculate $NOx_{nm}$ by taking the mean pollution over all the edges and dividing it by the total number of solutions we used, as we use one solution per hour. $NOx_{nm}$ helps us understand how well the pollution is distributed over an area. This is due to the insight that to decrease $NOx_{nm}$, we need to decrease the mean amount of $NOx$ released, or we have to increase the number of edges being used. Since we are releasing the same number of vehicles per hour for all the simulations, $NO_x{nm}$ will be highly dependent on the pollution spread.

For \emph{MaxFlow-MCMC}, we set $\lambda$ close to 1 since we measure the lengths of paths in meters. Given that the probability of our \emph{Markov chain} moving from one state to another depends on the difference in their length, and if we set lambda to a smaller value like $0.5$, the probability of moving to a state that is just $5$ meters longer would be close to $0.03$.

\begin{table}[t]
    \centering

    \resizebox{\columnwidth}{!}{
    \begin{tabular}{lrrrrrr}
    \toprule
        Map name & Sq. km & Edges & Nodes & $s,t$ Pairs & $k$ \\
        \midrule
        Small map & $0.03$ & 37 & 25 & 1 & 9\\
        Seattle & $25.16$ & 18699 & 14939 & 2 & 170 \\
        Berkeley & $82.90$ & 30808 & 24864  & 1 & 440\\
        Kanpur & $18.50$ & 29956 & 20707 & 1 & 350\\
        Islington & $14.90$ & 5382 & 2367 & 2 & 300\\
        
    \bottomrule    
    \end{tabular}}
    \caption{ Details of the maps used for the experiments. We use a small map to compare with \emph{k-PMFA} since it is infeasible to run for larger maps.} 
    \label{tab:map_details}
\end{table}

For comparison with previous work on \emph{k-PMFA} \citet{KAMISHETTY2020102194}, we needed to make some modifications since the algorithm was designed for directed graphs. Hence, it can become inefficient to distribute the pollution in road networks containing lanes in both directions. This is because solutions generated by it can include flow in both directions of the road as part of two different paths, e.g., a max flow solution can contain one path that has a flow of $1$ from node $p$ to node $q$, whereas another path can have a flow of $1$ from node $q$ to node $p$. This will lead to increased pollution with no effective increase in capacity, which can be avoided. Hence, we extend the \emph{k-PMFA} algorithm to undirected graphs by using \emph{Breadth-First Search (BFS)} on the flow graph of the max flow solutions to generate the individual paths. This can help avoid the above situation while retaining the max flow solutions where the flows do not cancel each other. 

Runtime results presented were measured on an Intel E5-2640 processor on an HP SL230s compute node. All the results were averaged over $30$ runs.

\begin{figure}[t]   \centering

    \resizebox{0.9\columnwidth}{!}{
    \includegraphics{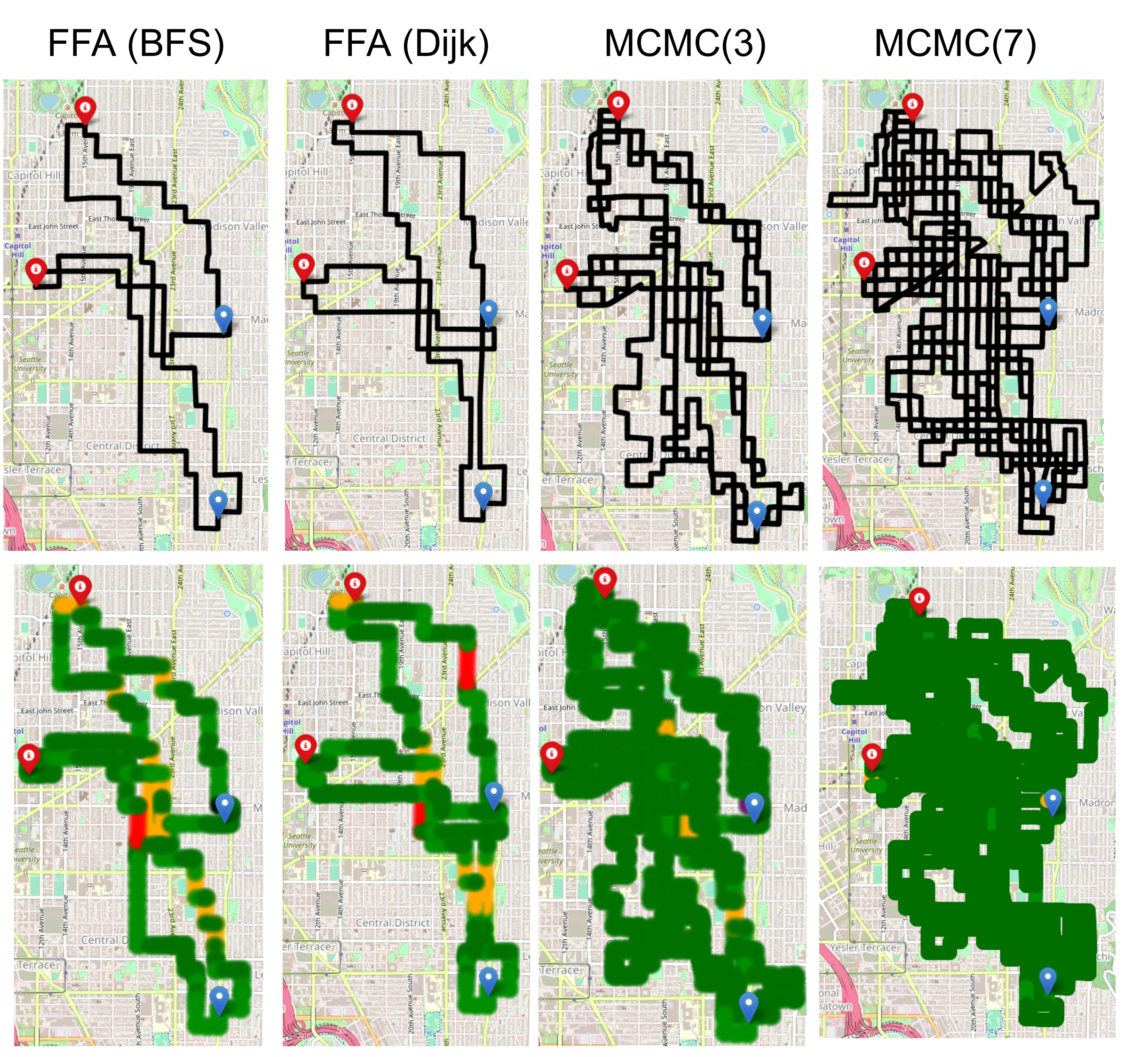}}
    \caption{ Maxflows (top) and pollution (bottom) generated for the different policies on the map of Seattle with two $s,t$-pairs. MCMC($k$) denotes routing using $k$ samples from \emph{Markov Chain}. \emph{MaxFlow-MCMC} with $7$ samples prevents severe pollution (red, orange) from happening in any area. Legend of colors is same as Figure \ref{fig:front-pic}.}
    \label{fig:heatmap1}
\end{figure}

\begin{table}[t]\centering
        \resizebox{\columnwidth}{!}{
        \begin{tabular}{llcrcrrrcr}\toprule
        
            \parbox[t]{2mm}{\multirow{2}{*}{\rotatebox[origin=c]{90}{City}}} &Traffic  &&Length  && \multicolumn{3}{c}{Pollution ($NOx_{nm}$ in mg)}\\
         \cmidrule{6-8}
         & Policy && avg (m) && avg & max  * $10^{-3}$ & total * $10^{-4}$\\
         \midrule     
     
     
        \parbox[t]{2mm}{\multirow{4}{*}{\rotatebox[origin=c]{90}{Berkeley}}}&FFA(BFS) && 8528 && 945 & 10.2 & 57.4\\
        &FFA(Dij) && 8001 && 718 & 8.5 & 54.3\\ 
        &MCMC(3) && 9457 $\pm$ 40 && 462 $\pm$ 4 & 6.7 $\pm$ 0.2 & 66.5 $\pm$ 0.4\\
        &MCMC(7) && 9456 $\pm$ 19 && 316 $\pm$ 2 & 5.7 $\pm$ 0.2 & 66.7 $\pm$ 0.2\\ 
        \midrule 
        \parbox[t]{2mm}{\multirow{4}{*}{\rotatebox[origin=c]{90}{Islington}}}&FFA(BFS) && 4723 && 2048 & 13.7 & 62.7\\
        &FFA(Dij) && 4930 && 2047 & 13.7 & 65.3\\
        &MCMC(3) && 4976 $\pm$ 10 && 1369 $\pm$ 5 & 8.0 $\pm$ 0.1 & 67.5 $\pm$ 0.2\\
        &MCMC(7) && 4945 $\pm$ 7 && 1150 $\pm$ 4 & 7.6 $\pm$ 0.0 & 67.0 $\pm$ 0.1 \\
    
        \midrule 
     
     
        \parbox[t]{2mm}{\multirow{4}{*}{\rotatebox[origin=c]{90}{Seattle}}}&FFA(BFS) && 3270 && 2402 & 18.2 & 57.6\\
        &FFA(Dij) && 2133 && 1999 & 10.0 &55.4\\
        &MCMC(3) && 3914 $\pm$ 27 && 856 $\pm$ 11 & 12.8 $\pm$ 0.9 & 77.3 $\pm$ 1.2\\
        &MCMC(7) && 4041 $\pm$ 24&& 512 $\pm$ 9 & 12.3 $\pm$ 0.8 & 83.6 $\pm$ 1.1\\
        
        \midrule 

        \parbox[t]{2mm}{\multirow{4}{*}{\rotatebox[origin=c]{90}{Kanpur}}}&FFA(BFS) && 4628 && 1152 & 8.8 &48.0\\
        &FFA(Dij) && 4248 &&  922 & 8.8 & 44.2\\
        &MCMC(3) && 4513 $\pm$ 4 && 597 $\pm$ 4 & 6.6 $\pm$ 0.1 & 50.6 $\pm$ 0.2\\
        &MCMC(7) && 4512 $\pm$ 4 && 444 $\pm$ 2 & 5.9 $\pm$ 0.1 & 50.7 $\pm$ 0.1\\
    
    
     
  
        \bottomrule
        \end{tabular}}
        \caption{Comparison of lengths and pollution generated for different traffic policies and cities simulated in \emph{SUMO}. Since our \emph{MaxFlow-MCMC} policy is randomized, we provide mean value and the standard error \cite{SEM} from 30 runs. Avg. pollution decreases significantly in all cases. Max pollution also reduces in most cases helping our objective of distributing the pollution better. There is some increase in distance travelled and hence the total pollution (summed up over all the links).}
          \label{tab:city-bench}
        \end{table}

\section{Results}
\label{sec:results}

We benchmark the performance of our \emph{Maxflow-MCMC} based traffic policy on large-scale maps of many cities along with baseline algorithms. We also compare it with previous work of \emph{k-PMFA} \cite{KAMISHETTY2020102194} on a small map due to its scalability issues. All the networks were obtained from \emph{OpenStreetMap} \cite{haklay2008openstreetmap}. 

\paragraph{Pollution reduction in City Scale Road Networks.}
    
We benchmarked \emph{MaxFlow-MCMC} on four larger real-world road networks, namely Seattle, Berkeley, Kanpur, and Islington, to demonstrate scalability (see Table \ref{tab:city-bench}). Table \ref{tab:map_details} provides details of the maps used. We used MaxFlow-MCMCLarge from Table \ref{tab:maxlfowmcmcproperties} for all the simulations. For Kanpur map, we set the value of $\lambda = 0.95$ since $\lambda = 0.99$ was giving us longer paths. We use two versions of \emph{FFA}, which use \emph{Breadth First Search (BFS)} and \emph{Dijkstra's} shortest path algorithm, respectively, to find the augmenting paths. The difference between the two versions of \emph{FFA} is that \emph{BFS} finds the shortest paths in terms of the number of edges and \emph{Dijkstra's} finds in terms of total path length. We also use two versions of \emph{MaxFlow-MCMC}, which samples $3$ and $7$ max flow solutions, respectively. 

In Table \ref{tab:city-bench}, the average pollution per link has decreased for all maps while there is some increase in the average distance traveled and total pollution summed up over all the edges. This is because the samples can be slightly longer than the shortest paths. However, the pollution is spread out because of the diverse samples, decreasing the max pollution in any link, which was our objective. It is also clear that increasing the number of max flow solutions sampled from $3$ to $7$ further prevents severe pollution. Figure \ref{fig:heatmap1} shows the NOx heatmap and the paths used by the algorithms for the multi-source multi-sink scenario of Seattle. The heatmap shows that while \emph{Maxflow-MCMC} distributes pollution over a much larger area, it also reduces the concentration of pollution in specific areas compared to \emph{FFA}. Plots for the other cities are provided in the Appendix.

\begin{table}[t]\centering
    \resizebox{\columnwidth}{!}{
        \begin{tabular}{llcrcrrrcr}\toprule
        
            \parbox[t]{2mm}{\multirow{2}{*}{\rotatebox[origin=c]{90}{City}}} &Traffic  &&Length  && \multicolumn{3}{c}{Pollution ($NOx_{nm}$ in mg)}\\
         \cmidrule{6-8}
         & Policy && avg (m) && avg & max & total\\
         \midrule     
     \parbox[t]{2mm}{\multirow{4}{*}{\rotatebox[origin=c]{90}{Small Map}}}&FFA(BFS) && 213 && 4387 & 9644 & 61428\\
     &FFA(Dij) && 198 && 4616 & 19884 & 73853 \\
     &k-PMFA(7) && 208 && 1853 & 11255 & 72262 \\
     &MCMC(7) && 207 && 1784  & 11776  & 69104\\
  
        \bottomrule
        \end{tabular}}
        \caption{Comparing the pollution values of \emph{$k$-PMFA} and \emph{MCMC} for the small map with $k =9$.}
          \label{tab:small-map}
        \end{table}
        
        
\paragraph{Comparison with k-PMFA.}


 \emph{k-PMFA} \cite{KAMISHETTY2020102194} is a pollution-aware routing algorithm that we compare within Table \ref{tab:small-map}. The table shows the performance of three algorithms, namely \emph{MaxFlow-MCMC}, \emph{k-PMFA}, and \emph{FFA}. As shown, the \emph{Maxflow-MCMC} gives comparable results while being scalable.

\begin{figure}[t]
    \centering
   \resizebox{0.8\columnwidth}{!}{
\includegraphics{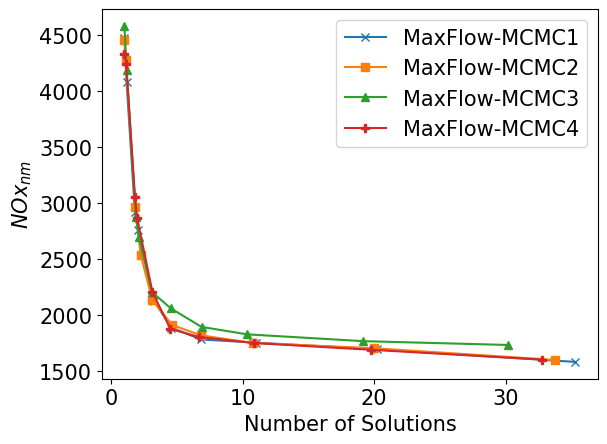}}
\caption{The change in normalised mean pollution reduces as we increase the number of solutions (obtained by increasing value of k). While there is significant reduction upto 10, the benefits are lesser later.}
 \label{meanpol}
\end{figure}

\begin{figure}[t]
    \centering
    \resizebox{0.8\columnwidth}{!}{

\includegraphics{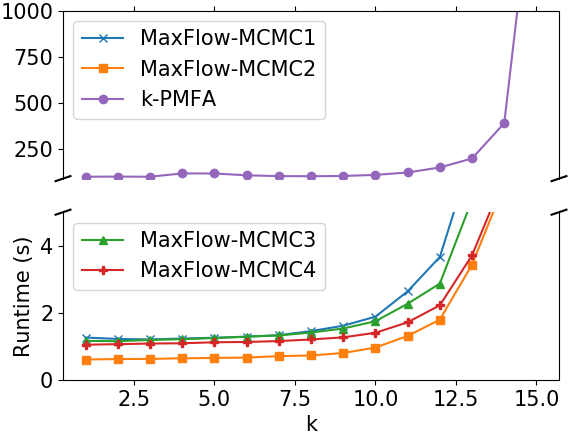}}
\caption{The runtime of \emph{k-PMFA} is two orders of magnitude higher than \emph{MCMC}, making it infeasible for road networks over a few sq km. Increasing the value of k increases the diversity while also increasing the running time, especially after $k=10$.}
\label{fig:runtime}

\end{figure}




\paragraph{Effect of hyperparameters.}


Figure~\ref{meanpol} studies the effect of using multiple max flow solutions on the pollution level. We also experiment with four different hyperparameter values of \emph{MaxFlow-MCMC} as given in Table \ref{tab:maxlfowmcmcproperties}, denoted as \emph{MaxFlow-MCMC} $1$-$4$.  It shows that increasing the number of solutions after $7$-$10$ does not significantly change the $NOx_{nm}$ value across the different sets of parameters tested. A key insight from this experiment is that for practical purposes, we can identify a set of, say $7$ solutions and use one solution per day of the week to route the vehicles and obtain a better distribution of pollution compared to the \emph{FFA} solution.

\paragraph{Runtime.}

Figure \ref{fig:runtime} shows the runtime comparison results between \emph{MaxFlow-MCMC} and \emph{k-PMFA}. Although \emph{MaxFlow-MCMC} provides approximately the same number of solutions as \emph{k-PMFA}, it is significantly faster. For example, for $k$ = $10$ the runtime needed for \emph{k-PMFA} is $110.07$ seconds while the slowest \emph{MCMC} version has a runtime of $1.89$ seconds. For values of $k$ from $1$ to $15$, the slowest version of \emph{MaxFlow-MCMC} (MaxFlow-MCMC1) showed a speedup of $65$x on average, and the fastest version (MaxFlow-MCMC2) showed a speedup of $130$x on average.

\section{Conclusion}
We design a \emph{Markov Chain} to sample integer max flow solutions for a planar graph. We used it to build an urban traffic routing policy that prevents the concentration of pollution. Since \emph{MCMC} algorithms are very efficient, we are able to scale our experiments to large cities. We believe that the sampling method for integer max flows can be of independent interest in solving other combinatorial optimization problems.
For future work, the policy can be expanded to non-planar graphs since real-world maps can contain flyovers and traffic between multiple sources and destinations.

\bibliography{aaai23}

\begin{thebibliography}{32}
\providecommand{\natexlab}[1]{#1}

\bibitem[{Ahmed et~al.(2020)Ahmed, Adnan, Janssens, and Wets}]{AHMED2020101965}
Ahmed, S.; Adnan, M.; Janssens, D.; and Wets, G. 2020.
\newblock A route to school informational intervention for air pollution
  exposure reduction.
\newblock \emph{Sustainable Cities and Society}, 53: 101965.

\bibitem[{Alam, Perugu, and McNabola(2018)}]{ALAM201882}
Alam, M.; Perugu, H.; and McNabola, A. 2018.
\newblock A comparison of route-choice navigation across air pollution
  exposure, CO2 emission and traditional travel cost factors.
\newblock \emph{Transportation Research Part D: Transport and Environment}, 65:
  82--100.

\bibitem[{Barrett et~al.(2008)Barrett, Bisset, Holzer, Konjevod, Marathe, and
  Wagner}]{barrett2008engineering}
Barrett, C.; Bisset, K.; Holzer, M.; Konjevod, G.; Marathe, M.; and Wagner, D.
  2008.
\newblock Engineering label-constrained shortest-path algorithms.
\newblock In \emph{International conference on algorithmic applications in
  management}, 27--37. Springer.

\bibitem[{Boriboonsomsin et~al.(2012)Boriboonsomsin, Barth, Zhu, and
  Vu}]{boriboonsomsin2012eco}
Boriboonsomsin, K.; Barth, M.~J.; Zhu, W.; and Vu, A. 2012.
\newblock Eco-routing navigation system based on multisource historical and
  real-time traffic information.
\newblock \emph{IEEE Transactions on Intelligent Transportation Systems},
  13(4): 1694--1704.

\bibitem[{Borradaile et~al.(2011)Borradaile, Klein, Mozes, Nussbaum, and
  Wulff-Nilsen}]{6108162}
Borradaile, G.; Klein, P.~N.; Mozes, S.; Nussbaum, Y.; and Wulff-Nilsen, C.
  2011.
\newblock Multiple-Source Multiple-Sink Maximum Flow in Directed Planar Graphs
  in Near-Linear Time.
\newblock In \emph{2011 IEEE 52nd Annual Symposium on Foundations of Computer
  Science}, 170--179.

\bibitem[{Bubley(2001)}]{Bubley}
Bubley, R. 2001.
\newblock \emph{Randomized algorithms - approximation, generation and
  counting}.
\newblock Distinguished dissertations. Springer.
\newblock ISBN 978-1-85233-325-6.

\bibitem[{Diaconis(2009)}]{Diaconis}
Diaconis, P. 2009.
\newblock The Markov chain Monte Carlo revolution.
\newblock \emph{Bull. Amer. Math. Soc. 46, 179-205}.

\bibitem[{Dicker(2021)}]{3newways27:online}
Dicker, R. 2021.
\newblock 3 new ways to navigate more sustainably with Maps.
\newblock
  \url{https://blog.google/products/maps/3-new-ways-navigate-more-sustainably-maps/}.
\newblock (Accessed on 07/24/2022).

\bibitem[{Durrett(2019)}]{durrett2019probability}
Durrett, R. 2019.
\newblock \emph{Probability: theory and examples}, volume~49.
\newblock Cambridge university press.

\bibitem[{Easley, Kleinberg et~al.(2012)}]{easley2012networks}
Easley, D.; Kleinberg, J.; et~al. 2012.
\newblock Networks, crowds, and markets: Reasoning about a highly connected
  world.
\newblock \emph{Significance}, 9(1): 43--44.

\bibitem[{Erickson(1999)}]{erickson1999algorithms}
Erickson, J. 1999.
\newblock Algorithms.

\bibitem[{Ford and Fulkerson(1956)}]{ford1956maximal}
Ford, L.~R.; and Fulkerson, D.~R. 1956.
\newblock Maximal flow through a network.
\newblock \emph{Canadian journal of Mathematics}, 8: 399--404.

\bibitem[{Gualtieri et~al.(2015)Gualtieri, Crisci, Tartaglia, Toscano, and
  Gioli}]{gualtieri2015statistical}
Gualtieri, G.; Crisci, A.; Tartaglia, M.; Toscano, P.; and Gioli, B. 2015.
\newblock A statistical model to assess air quality levels at urban sites.
\newblock \emph{Water, Air, \& Soil Pollution}, 226(12): 1--15.

\bibitem[{{Gubernatis}(2005)}]{Rosenbluth}
{Gubernatis}, J.~E. 2005.
\newblock {Marshall Rosenbluth and the Metropolis algorithm)}.
\newblock \emph{Physics of Plasmas}, 12(5): 057303.

\bibitem[{Haklay and Weber(2008)}]{haklay2008openstreetmap}
Haklay, M.; and Weber, P. 2008.
\newblock OpenStreetMap: User-Generated Street Maps.
\newblock \emph{Pervasive Computing}, 7(4): 12--18.

\bibitem[{Hastings(1970)}]{Hastings}
Hastings, W.~K. 1970.
\newblock {Monte Carlo sampling methods using Markov chains and their
  applications}.
\newblock \emph{Biometrika}, 57(1): 97--109.

\bibitem[{Hunter(2007)}]{Hunter:2007}
Hunter, J.~D. 2007.
\newblock Matplotlib: A 2D graphics environment.
\newblock \emph{Computing in Science \& Engineering}, 9(3): 90--95.

\bibitem[{Jerrum, Valiant, and Vazirani(1986)}]{JVV1986}
Jerrum, M.; Valiant, L.; and Vazirani, V. 1986.
\newblock Random Generation of Combinatorial Structures from a Uniform
  Distribution.
\newblock \emph{Theor. Comput. Sci.}, 43: 169--188.

\bibitem[{Kamishetty, Vadlamannati, and Paruchuri(2020)}]{KAMISHETTY2020102194}
Kamishetty, S.; Vadlamannati, S.; and Paruchuri, P. 2020.
\newblock Towards a better management of urban traffic pollution using a Pareto
  max flow approach.
\newblock \emph{Transportation Research Part D: Transport and Environment}, 79:
  102194.

\bibitem[{Lopez et~al.(2018)Lopez, Behrisch, Bieker-Walz, Erdmann,
  Fl{\"o}tter{\"o}d, Hilbrich, L{\"u}cken, Rummel, Wagner, and
  Wie{\ss}ner}]{SUMO2018}
Lopez, P.~A.; Behrisch, M.; Bieker-Walz, L.; Erdmann, J.; Fl{\"o}tter{\"o}d,
  Y.-P.; Hilbrich, R.; L{\"u}cken, L.; Rummel, J.; Wagner, P.; and Wie{\ss}ner,
  E. 2018.
\newblock Microscopic Traffic Simulation using SUMO.
\newblock In \emph{The 21st IEEE International Conference on Intelligent
  Transportation Systems}. IEEE.

\bibitem[{Martin and Randall(2000)}]{MartinRandall}
Martin, R.; and Randall, D. 2000.
\newblock Sampling adsorbing staircase walks using a new Markov chain
  decomposition method.
\newblock In \emph{Proceedings 41st Annual Symposium on Foundations of Computer
  Science}, 492--502.

\bibitem[{Montanari and Penna(2015)}]{MontanariPenna}
Montanari, S.; and Penna, P. 2015.
\newblock On Sampling Simple Paths in Planar Graphs According to Their Lengths.
\newblock In Italiano, G.~F.; Pighizzini, G.; and Sannella, D.~T., eds.,
  \emph{Mathematical Foundations of Computer Science 2015}, 493--504. Berlin,
  Heidelberg: Springer Berlin Heidelberg.
\newblock ISBN 978-3-662-48054-0.

\bibitem[{Pearce and Tambe(2007)}]{pearce2007quality}
Pearce, J.~P.; and Tambe, M. 2007.
\newblock Quality Guarantees on k-Optimal Solutions for Distributed Constraint
  Optimization Problems.
\newblock In \emph{IJCAI}, 1446--1451.

\bibitem[{Pope~III and Dockery(2006)}]{pope2006health}
Pope~III, C.~A.; and Dockery, D.~W. 2006.
\newblock Health effects of fine particulate air pollution: lines that connect.
\newblock \emph{Journal of the air \& waste management association}, 56(6):
  709--742.

\bibitem[{Schneider and Sbalzarini(2019)}]{schneider2019finding}
Schneider, S.; and Sbalzarini, I.~F. 2019.
\newblock Finding faces in a planar embedding of a graph.

\bibitem[{Schrijver(2002)}]{schrijver2002history}
Schrijver, A. 2002.
\newblock On the history of the transportation and maximum flow problems.
\newblock \emph{Mathematical programming}, 91(3): 437--445.

\bibitem[{Siddharth~Kalla(2022)}]{SEM}
Siddharth~Kalla, L. T.~W. 2022.
\newblock Standard Error of the Mean.
\newblock \url{https://explorable.com/standard-error-of-the-mean}.

\bibitem[{{The Sage Developers}(2021)}]{sagemath}
{The Sage Developers}. 2021.
\newblock \emph{{S}ageMath, the {S}age {M}athematics {S}oftware {S}ystem
  ({V}ersion 9.4)}.
\newblock {\tt https://www.sagemath.org}.

\bibitem[{Tom{\`a}s~Verg{\'e}s(2013)}]{tomas2013analysis}
Tom{\`a}s~Verg{\'e}s, J. 2013.
\newblock \emph{Analysis and simulation of traffic management actions for
  traffic emission reduction}.
\newblock Ph.D. thesis, TU Berlin.

\bibitem[{Trudeau(2013)}]{trudeau2013introduction}
Trudeau, R.~J. 2013.
\newblock \emph{Introduction to graph theory}.
\newblock Courier Corporation.

\bibitem[{Valiant(1979)}]{valiant1979complexity}
Valiant, L.~G. 1979.
\newblock The complexity of enumeration and reliability problems.
\newblock \emph{SIAM Journal on Computing}, 8(3): 410--421.

\bibitem[{{World Health Organization} et~al.(2016)}]{world2016ambient}
{World Health Organization}; et~al. 2016.
\newblock Ambient air pollution: A global assessment of exposure and burden of
  disease.

\end{thebibliography}

\newpage

\section{Appendix}
\subsection{Proof of Stationary Distribution}

\begin{lemma}
\emph{A stationary distribution of $M_{\text{flow}}$ is
$$ \pi(x) = \frac{\lambda^{|x|}}{Z} \qquad \text{ where } \qquad Z = \sum_{y \in \Omega}\lambda^{|y|}$$
}
\end{lemma}
\begin{proof}
We show that $\pi$ satisfies the detailed balance condition. That is, for the transition probabilities $P(x,y)$ (from $x$ to $y$), and  $ \forall x,y \in \Omega , \pi(x) \cdot P(x,y) =\pi(y) \cdot P(y,x) $.

For the case where $x$ and $y$ are states that differ by more than one face, $P(x,y)=0$ according to our transition rules, and therefore the detailed balance condition will be satisfied. Let us take two states, x, and y, which differ only by one face. Let $P(x,y)$ be the probability of transitioning from $x$ to $y$. Let $F$ be the total number of faces in the graph. Now,
$$
\pi(x) \cdot P(x,y) = \frac{\lambda^{|x|}}{Z} \cdot \frac{1}{F}\cdot \frac{1}{mf} \cdot \min\left\{1,\frac{\lambda^{|y|}}{\lambda^{|x|}}\right\}
$$

Similarly,
$$
\pi(y)\cdot P(y,x) = \frac{\lambda^{|y|}}{Z} \cdot \frac{1}{F}\cdot \frac{1}{mf} \cdot \min\left\{1,\frac{\lambda^{|x|}}{\lambda^{|y|}}\right\}
$$

If $\frac{\lambda^{|y|}}{\lambda^{|x|}} \geq 1$,
\begin{align*}
\pi(x) \cdot P(x,y) &=  \frac{\lambda^{|x|}}{Z} \cdot \frac{1}{F}\cdot \frac{1}{mf} \cdot 1 =  \frac{\lambda^{|y|}}{Z} \cdot \frac{1}{F}\cdot \frac{1}{mf} \cdot \frac{\lambda^{|x|}}{\lambda^{|y|}}\\ 
&= \pi(y) \cdot P(y,x)      
\end{align*}

If $\frac{\lambda^{|y|}}{\lambda^{|x|}} < 1$,
\begin{align*}
\pi(x) \cdot P(x,y) &=  \frac{\lambda^{|x|}}{Z} \cdot \frac{1}{F}\cdot \frac{1}{mf} \cdot \frac{\lambda^{|y|}}{\lambda^{|x|}} =  \frac{\lambda^{|y|}}{Z} \cdot \frac{1}{F}\cdot \frac{1}{mf} \cdot 1\\ 
&= \pi(y) \cdot P(y,x)      
\end{align*}
\end{proof}

\subsection{Diagrammatical explanation of proof of Proposition 3.3}

In Figure \ref{fig:diag_proof}, The first two diagrams represent the two top outer paths. The common nodes between the two are $s,1,4,7,t$, which gives us $4$ subpaths. For each of these subpaths, we convert the subpath in the lower one of the two to the other, resulting in both of these paths becoming the path shown in the third diagram. Next create the residual graph and repeat the steps until both the max flow solutions reach the same state. Since all our steps are reversible, we can reach any state from any other state in a finite number of steps
to the other.
\begin{figure}
    \centering
    \resizebox{\columnwidth}{!}{
    \includegraphics{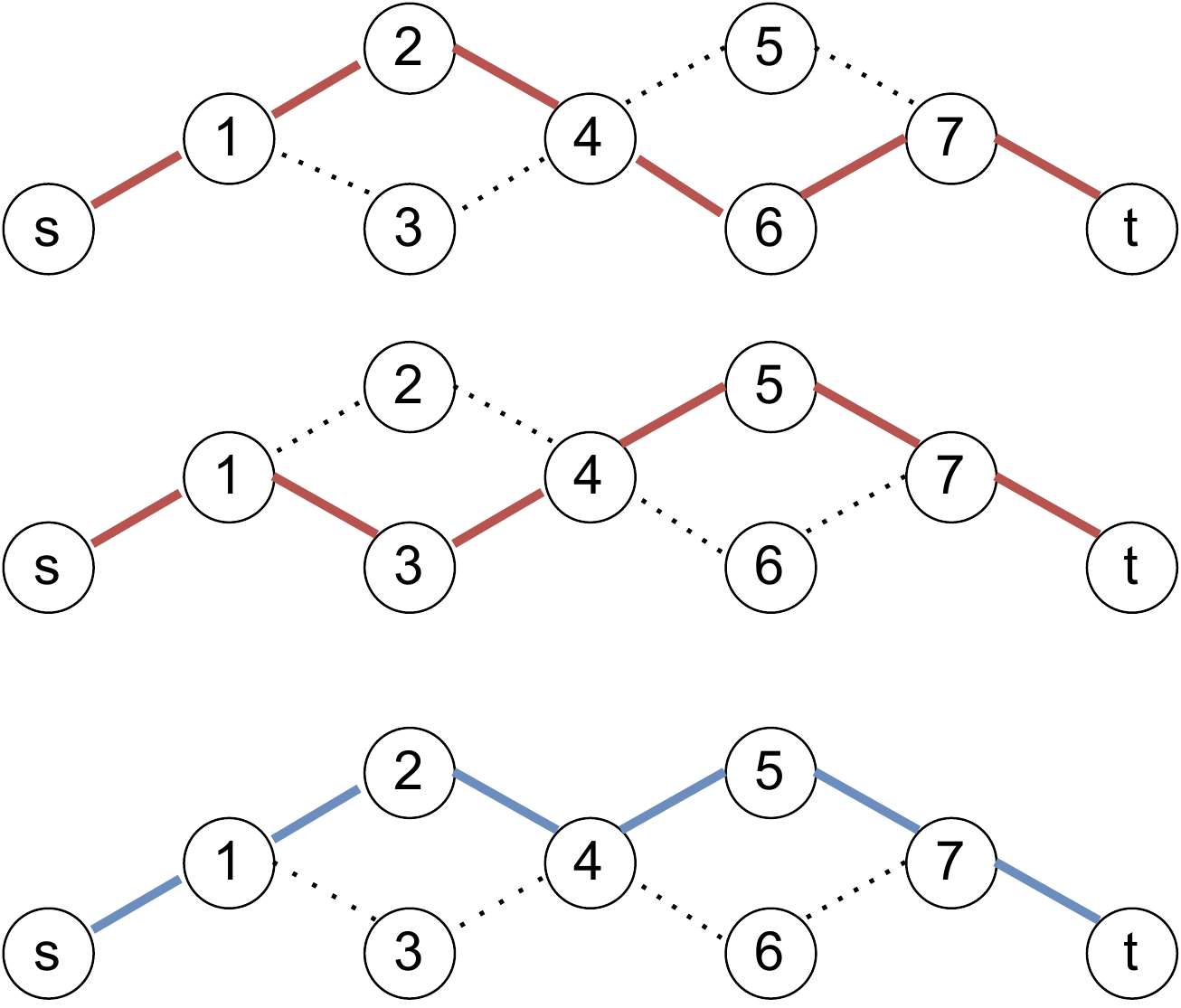}}
    \caption{Diagrammatical explanation of proof of Proposition 3.3}
    \label{fig:diag_proof}
\end{figure}
\subsection{Additional analysis with k-PMFA}

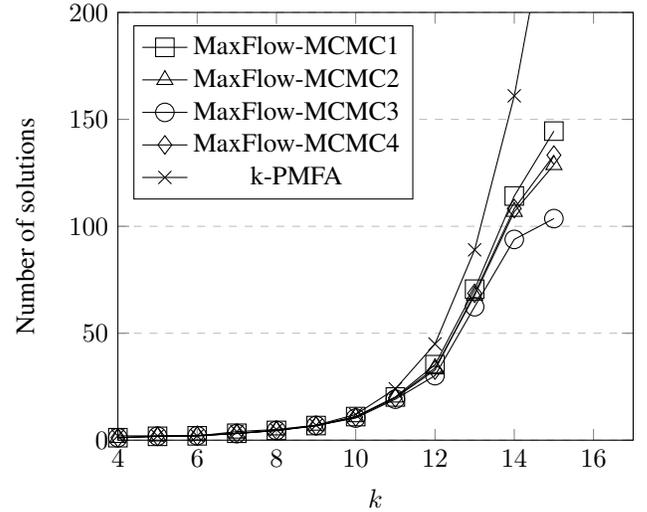
\begin{figure}[H]
    \centering
    \resizebox{\columnwidth}{!}{
\begin{tikzpicture}
\begin{axis}[
    xmin=4, xmax=17,
    ymin=0, ymax=200,
    legend pos=north west,
    ymajorgrids=true,
    grid style=dashed,
    xlabel={$k$},
    ylabel={Number of solutions}
]
\addplot[
    mark=square,mark options={scale=1.75}
    ]
    coordinates {
(1,1)
(2,1)
(3,1)
(4,1.166666667)
(5,1.8)
(6,2.033333333)
(7,3.233333333)
(8,4.566666667)
(9,6.833333333)
(10,11.03333333)
(11,20.2)
(12,35.3)
(13,70.6)
(14,114.1666667)
(15,144.5)    };

\addplot[
    mark=triangle,mark options={scale=1.75}
    ]
    coordinates {
(1,1)
(2,1)
(3,1)
(4,1.1)
(5,1.833333333)
(6,2.233333333)
(7,3.133333333)
(8,4.633333333)
(9,6.833333333)
(10,10.8)
(11,20)
(12,33.73333333)
(13,67.66666667)
(14,106.8333333)
(15,128.7)
    };

\addplot[
    mark=o,mark options={scale=1.75}
    ]
    coordinates {
(1,1)
(2,1)
(3,1)
(4,1.166666667)
(5,1.866666667)
(6,2.1)
(7,3.1)
(8,4.533333333)
(9,6.866666667)
(10,10.33333333)
(11,19.16666667)
(12,30.16666667)
(13,62.4)
(14,93.86666667)
(15,103.6)
    };
    
\addplot[
    mark=diamond,mark options={scale=1.75}
    ]
    coordinates {
(1,1)
(2,1)
(3,1)
(4,1.133333333)
(5,1.8)
(6,1.966666667)
(7,3.133333333)
(8,4.5)
(9,6.666666667)
(10,10.83333333)
(11,19.76666667)
(12,32.8)
(13,68.66666667)
(14,108.3)
(15,133.2666667)
    };
    
\addplot[
    mark=x,mark options={scale=1.75}
    ]
    coordinates {
    (1,1)(2,1)(3,1)(4,2)(5,2)(6,2)(7,4)(8,5)(9,7)(10,12)(11,24)(12,45)(13,89)(14,161)(15,261)};
    \legend{MaxFlow-MCMC1,MaxFlow-MCMC2,MaxFlow-MCMC3,MaxFlow-MCMC4,k-PMFA}

\end{axis}
\end{tikzpicture}}
\caption{k vs Number of solutions: x-axis represents the maximum number of common edges between two max flow solutions and y-axis represents the size of the k-optimal solution set. Parameters for the MaxFlow-MCMC$1-4$ are provided in Table 1 of the paper. At k=0, size of the solution set will be 1 since no common edges are allowed between two max flow solutions. Number of max flow solutions increase as we increase k. Plot also shows that the number of solutions provided by MaxFlow-MCMC is very close to that of k-PMFA when the number of solutions is small. Varying the parameters of MaxFlow-MCMC changes the size of the k-optimal set obtained}
\label{numsolvsk}

\end{figure}

\subsection{Heatmaps of Cities}

\begin{figure}[H]   \centering
    Berkeley\\
    \resizebox{0.75 \columnwidth}{!}{
    \includegraphics{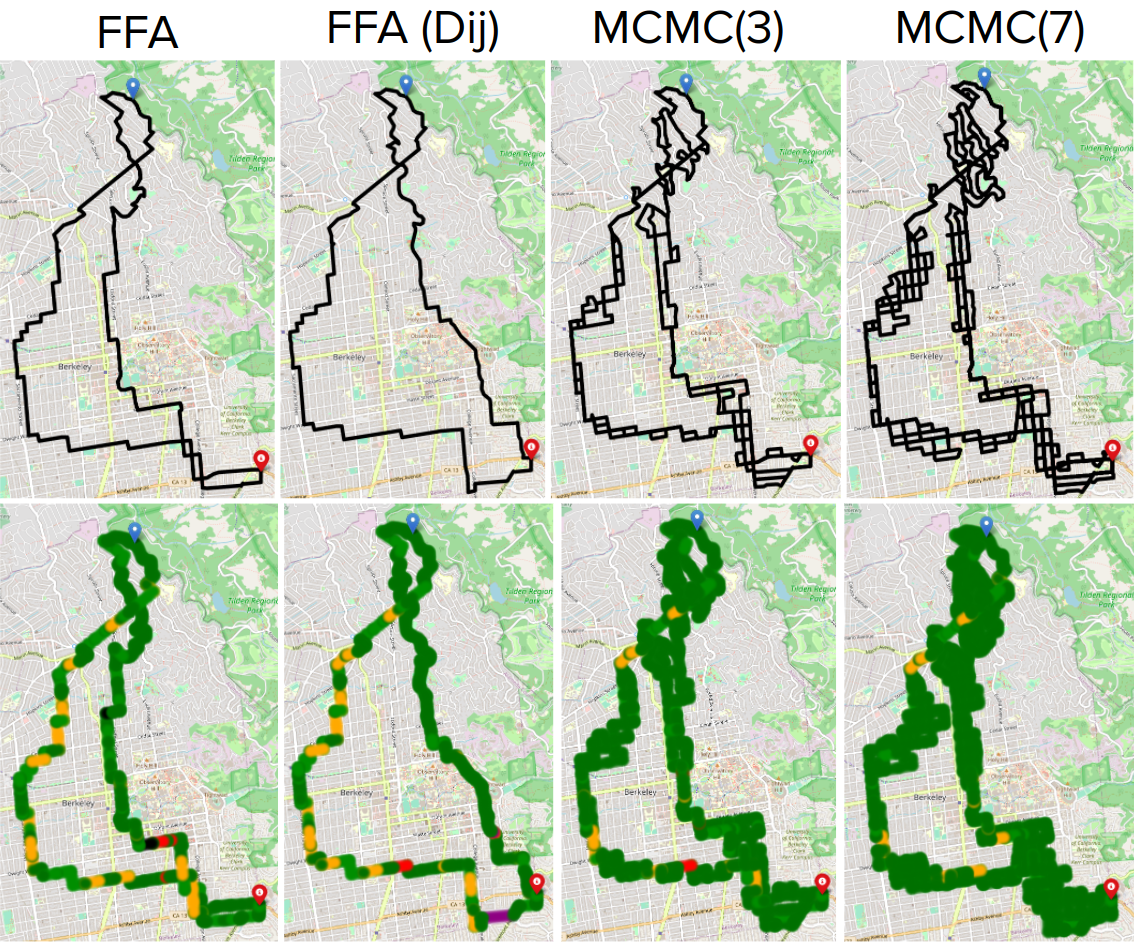}}\\
    Islington\\

    \resizebox{0.75 \columnwidth}{!}{
    \includegraphics{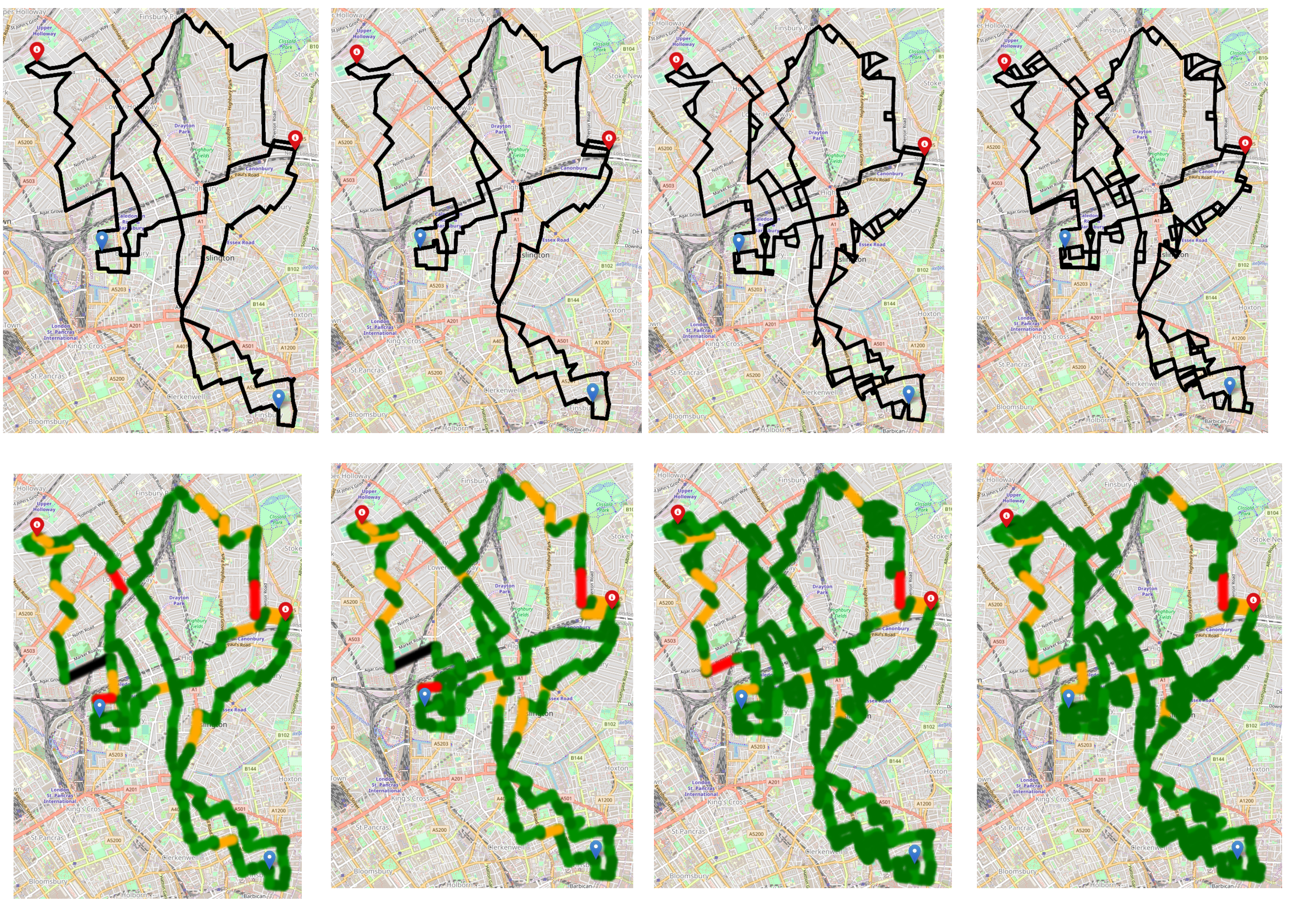}}\\
    Kanpur\\
\resizebox{0.75 \columnwidth}{!}{
    \includegraphics{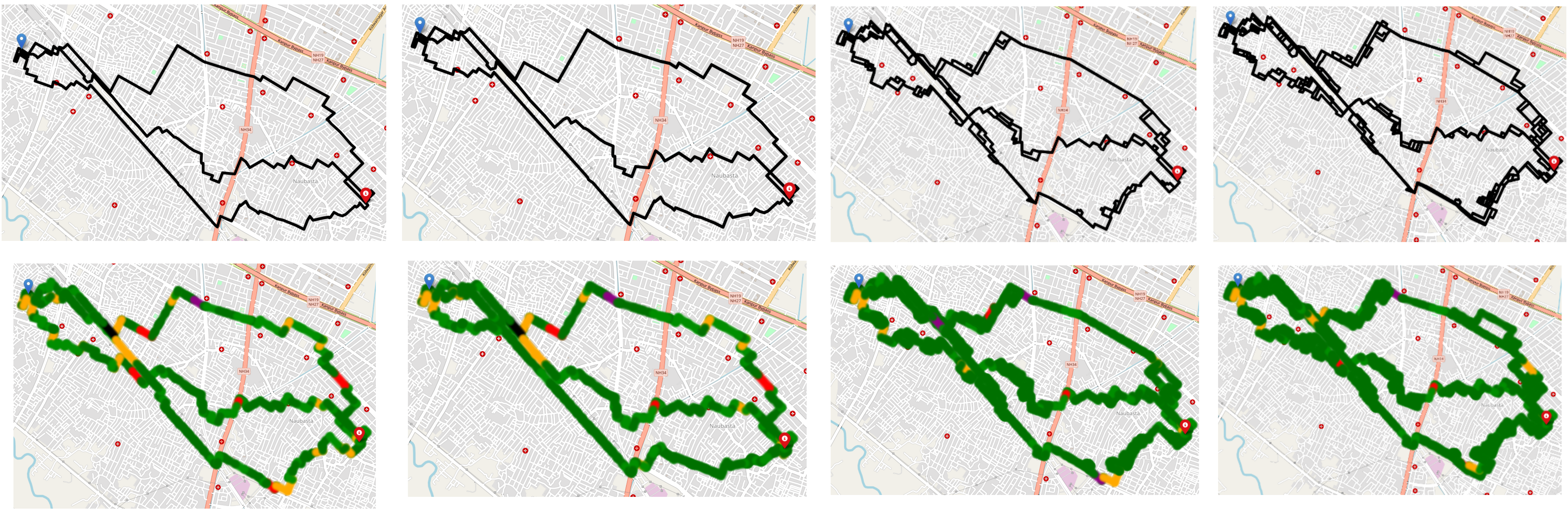}}
    
    \caption{Maxflows (top) and pollution (bottom) generated for the different policies on the maps of Berkeley (USA), Islington (UK) and Kanpur (India). The figures correspond to the results mentioned in Table 3 of the paper. MCMC($n$) denotes routing using $n$ samples from \emph{Markov Chain}. We can see that \emph{MaxFlow-MCMC} has reduced the areas severe pollution (purple,black) for all the maps in the figure. 
    Legend of colors is same as Figure 4 in the paper.}
    \label{fig:heatmap2}
\end{figure}

\end{document}